\documentclass{article}
\usepackage{pdfpages}

\usepackage[final]{nips_2016}

\usepackage[utf8]{inputenc} 
\usepackage[T1]{fontenc}    
\usepackage{hyperref}       
\usepackage{url}            
\usepackage{booktabs}       
\usepackage{amsfonts}       
\usepackage{nicefrac}       
\usepackage{microtype}      

\usepackage{times}
\usepackage{graphicx} 
\usepackage{subfigure}
\usepackage{wrapfig}

\usepackage{natbib}

\usepackage{algorithm}
\usepackage{algorithmic}

\usepackage{verbatim}
\usepackage{amsmath, amssymb,amsthm, color}
\usepackage{esvect}
\usepackage{bbm}
\usepackage{dsfont}
\usepackage{tikz}
\usepackage{relsize}
\usepackage{adjustbox}
\usepackage{pdflscape}
\usepackage{caption}
\usetikzlibrary{arrows,decorations.pathmorphing,fit,positioning}
\renewcommand{\algorithmicrequire}{\textbf{Input:}}
\renewcommand{\algorithmicensure}{\textbf{Output:}}

\def\P{\mathbb P}
\def\b{\boldsymbol{\beta}}
\def\t{\boldsymbol{\theta}}

\DeclareMathOperator*{\argmin}{argmin}
\DeclareMathOperator*{\diag}{diag}
\DeclareMathOperator*{\rank}{rank}
\DeclareMathOperator*{\spa}{span}
\DeclareMathOperator*{\conv}{Conv}

\newsavebox\myboxA
\newsavebox\myboxB
\newlength\mylenA
\makeatletter
\newcommand*\xoverline[2][0.75]{%
    \sbox{\myboxA}{$\m@th#2$}%
    \setbox\myboxB\null
    \ht\myboxB=\ht\myboxA%
    \dp\myboxB=\dp\myboxA%
    \wd\myboxB=#1\wd\myboxA
    \sbox\myboxB{$\m@th\overline{\copy\myboxB}$}
    \setlength\mylenA{\the\wd\myboxA}
    \addtolength\mylenA{-\the\wd\myboxB}%
    \ifdim\wd\myboxB<\wd\myboxA%
       \rlap{\hskip 0.5\mylenA\usebox\myboxB}{\usebox\myboxA}%
    \else
        \hskip -0.5\mylenA\rlap{\usebox\myboxA}{\hskip 0.5\mylenA\usebox\myboxB}%
    \fi}
\makeatother
\def\W{\xoverline{W}}

\newcommand{\tr}{\operatorname{tr}}

\title{Geometric Dirichlet Means algorithm \\
for topic inference}

\author{
  Mikhail Yurochkin \\
  Department of Statistics\\
  University of Michigan\\
  \texttt{moonfolk@umich.edu} \\
  \And
  XuanLong Nguyen \\
  Department of Statistics\\
  University of Michigan\\
  \texttt{xuanlong@umich.edu} \\
}

\begin{document}

\theoremstyle{definition}
\newtheorem{defn}{Definition}
\newtheorem{thm}{Theorem}
\newtheorem{prop}{Proposition}
\newtheorem{cor}{Corollary}
\newtheorem{problem}{Problem}
\newtheorem{lem}{Lemma}
\newtheoremstyle{TheoremNum}
    {\topsep}{\topsep}              
    {\itshape}                      
    {}                              
    {\bfseries}                     
    {.}                             
    { }                             
    {\thmname{#1}\thmnote{ \bfseries #3}}
\theoremstyle{TheoremNum}
\newtheorem{pmn}{Problem}
\newtheorem{propn}{Proposition}

\maketitle

\begin{abstract}
We propose a geometric algorithm for topic learning and inference
that is built on the convex geometry of topics arising from the
Latent Dirichlet Allocation (LDA) model and its nonparametric extensions.
To this end we study the optimization of a geometric loss function,
which is a surrogate to the LDA's likelihood. Our method involves
a fast optimization based weighted clustering procedure augmented
with geometric corrections, which overcomes the computational and statistical inefficiencies encountered by
other techniques based on Gibbs sampling and variational inference, while achieving
the accuracy comparable to that of a Gibbs sampler.
The topic estimates produced by our method
are shown to be statistically consistent under some conditions.
The algorithm is evaluated with
extensive experiments on simulated and real data.
\end{abstract}

\section{Introduction}
Most learning and inference algorithms in the probabilistic
topic modeling literature can be delineated along two major lines:
the variational approximation popularized in the seminal paper
of~\citet{blei2003latent}, and the sampling based
approach studied by~\citet{pritchard2000inference} and other authors.
Both classes of inference algorithms, their
virtues notwithstanding, are known to exhibit certain deficiencies,
which can be traced back to the need for approximating or sampling from the
posterior distributions of the latent variables representing the
topic labels. Since these latent variables are not geometrically
intrinsic --- any permutation of the labels yields
the same likelihood --- the manipulation of these
redundant quantities tend to slow down the computation, and compromise
with the learning accuracy.

In this paper we take a convex
geometric perspective of the Latent Dirichlet Allocation, which
may be obtained by integrating out the latent topic label variables. As a result,
topic learning and inference may be formulated as a
convex geometric problem: the observed documents correspond to points
randomly drawn from a \emph{topic polytope}, a convex set whose vertices
represent the topics to be inferred. The original paper of~\citet{blei2003latent}
(see also \citet{hofmann1999probabilistic}) contains early hints about
a convex geometric viewpoint, which is left unexplored.
This viewpoint had laid dormant for quite some time, until
studied in depth in the work of Nguyen and co-workers, who
investigated posterior contraction behaviors for the LDA both
theoretically and practically~\citep{nguyen2015posterior,tang2014understanding}.

Another fruitful perspective on topic modeling can be obtained by partially
stripping away the distributional properties of the probabilistic model and
turning the estimation problem into a form of matrix factorization
\citep{deerwester1990indexing,xu2003document,anandkumar2012spectral,arora2012practical}.
We call this the linear subspace viewpoint. For instance,
the Latent Semantic Analysis approach \citep{deerwester1990indexing}, which can
be viewed as a precursor of the LDA model, looks to find a latent subspace via
singular-value decomposition, but has no topic structure.
Notably, the RecoverKL by \citet{arora2012practical} is one of the recent
fast algorithms with provable guarantees coming from the linear subspace perspective.

The geometric perspective continues to be the main force driving this work.
We develop and analyze a new class of algorithms for topic inference,
which exploits both the convex geometry of topic models and the distributional properties
they carry. The main contributions in this work are the following:
(i) we investigate a geometric loss function to be optimized,
which can be viewed as a surrogate to the LDA's likelihood;
this leads to a novel
estimation and inference algorithm --- the Geometric Dirichlet Means algorithm,
which builds upon a weighted k-means clustering procedure
and is augmented with a
geometric correction for obtaining polytope estimates; (ii)
we prove that the GDM algorithm is consistent, under conditions on the Dirichlet
distribution and the geometry of the topic polytope;
(iii) we propose a nonparametric extension of GDM and discuss geometric treatments
for some of the LDA extensions; (v) finally we provide a thorough evaluation
of our method against a Gibbs sampler, a variational algorithm, and the
RecoverKL algorithm. Our method is shown to be
comparable to a Gibbs sampler in terms of estimation accuracy, but
much more efficient in runtime. It outperforms RecoverKL algorithm
in terms of accuracy,
in some realistic settings of simulations and in real data.

The paper proceeds as follows.
Section \ref{bg} provides a brief background of the LDA and its convex geometric
formulation. Section \ref{geo} carries out the contributions outlined above.
Section \ref{pfm} presents experiments results.
We conclude with a discussion in Section \ref{dis}.

\section{Background on topic models}
\label{bg}
In this section we give an overview of the well-known Latent Dirichlet
Allocation model for topic modeling~\citep{blei2003latent}, and the geometry
it entails. Let $\alpha\in \mathbb{R}_+^K$ and $\eta\in \mathbb{R}_+^V$ be
hyperparameters, where $V$ denotes the number of words in a vocabulary, and $K$ the number of
topics. The $K$ topics are represented as distributions on words:
$\beta_k|\eta\,\thicksim\, \text{Dir}_V(\eta)$, for $k=1,\ldots, K$.
Each of the $M$ documents can be generated as follows. First, draw the document topic proportions:
$\theta_m|\alpha\,\thicksim\,\text{Dir}_K(\alpha)$, for $m=1,\ldots, M$. Next,
for each of the $N_m$ words in document $m$, pick a topic label $z$ and then sample a word $d$ from the chosen topic:
\begin{eqnarray}
z_{n_m}|\theta_m\,& \thicksim & \,\text{Categorical}(\theta_m);\,\,d_{n_m}|z_{n_m},\beta_{1\ldots K} \, \thicksim \,\text{Categorical}(\beta_{z_{n_m}}).
\end{eqnarray}
Each of the resulting documents is a vector of length $N_m$ with entries $d_{n_m}\in\{1,\ldots,V\}$, where $n_m=1,\ldots,N_m$.
Because these words are exchangeable by the modeling, they are equivalently represented as a vector of
word counts $w_m \in \mathbb{N}^V$. In practice, the Dirichlet distributions are often simplified
to be symmetric Dirichlet, in which case hyperparameters $\alpha, \eta\in \mathbb{R}_+$ and
we will proceed with this setting.
Two most common approaches for inference with the LDA are Gibbs sampling \citep{griffiths2004finding},
based on the Multinomial-Dirichlet conjugacy, and mean-field inference \citep{blei2003latent}. The former
approach produces more accurate estimates but is less computationally efficient than the latter.
The inefficiency of both techniques can be traced to the need for sampling or estimating the
(redundant) topic labels. These labels are not intrinsic --- any permutation of the
topic labels yield the same likelihood function.
\paragraph{Convex geometry of topics.}
By integrating out the latent variables that represent the topic labels,
we obtain a geometric formulation of the LDA.
Indeed, integrating $z$'s out yields that, for $m=1,\ldots, M$,
\[w_m|\theta_m,\beta_{1\ldots K}, N_m\,\thicksim\,\text{Multinomial}(p_{m1},\ldots,p_{mV},N_m),\]
where $p_{mi}$ denotes probability of observing the $i$-th word from the vocabulary in
the $m$-th document, and is given by
\begin{equation}
\label{pmi}
p_{mi} = \sum_{k=1}^K \theta_{mk}\beta_{ki}\text{ for }i=1,\ldots,V;\,m=1,\ldots,M.
\end{equation}
The model's geometry becomes clear.
Each topic is represented by a point $\beta_k$ lying in the $V-1$ dimensional probability
simplex $\Delta^{V-1}$.
Let $B:=\conv(\beta_1,\ldots,\beta_K)$ be the convex hull of the $K$ topics $\beta_k$, then
each document corresponds to a point $p_{m}:= (p_{m1},\ldots,p_{mV})$ lying inside the polytope $B$.
This point of view has been proposed before \citep{hofmann1999probabilistic}, although topic proportions $\theta$ were not given any geometric meaning. The following treatment of $\theta$ lets us relate to the LDA's Dirichlet prior assumption and complete the geometric perspective of the problem.
The Dirichlet distribution generates probability vectors $\theta_m$, which can be viewed
as the (random) \emph{barycentric coordinates} of the document $m$ with respect to the polytope $B$.
Each $p_m=\sum_k \theta_{mk} \beta_k$ is a vector of cartesian coordinates of the $m$-th document's
multinomial probabilities. Given $p_m$, document $m$ is generated by taking
$w_m\,\sim\,\text{Multinomial}(p_m, N_m)$. In Section \ref{pfm} we will show how this
interpretation of topic proportions can be utilized by other topic modeling approaches,
including for example the RecoverKL algorithm of \citet{arora2012practical}.
In the following the model geometry is exploited to derive fast and effective
geometric algorithm for inference and parameter estimation.

\section{Geometric inference of topics}
\label{geo}

We shall introduce a geometric loss function that can be viewed as a surrogate to
the LDA's likelihood. To begin, let $\b$ denote
the $K\times V$ topic matrix with rows $\beta_k$,
$\t$ be a $M\times K$ document topic proportions matrix with rows $\theta_m$,
and $\W$ be $M\times V$ normalized word counts matrix with rows $\bar{w}_m = w_m/N_m$.
\subsection{Geometric surrogate loss to the likelihood}
Unlike the original LDA formulation, here
the Dirichlet distribution on $\t$ can be viewed as a prior on parameters
$\t$.
The log-likelihood of the observed corpora of $M$ documents is
\[L(\t, \b) = \sum_{m=1}^M\sum_{i=1}^V w_{mi}\log{\left(\sum_{k=1}^K\theta_{mk}\beta_{ki}\right)},\]
where the parameters $\b$ and $\t$ are subject to constraints
$\sum_i \beta_{ki} = 1$ for each $k=1,\ldots,K$, and
$\sum_k \theta_{mk} = 1$ for each $m=1,\ldots,M$.
Partially relaxing these constraints and keeping only the one that
the sum of all entries for each row of the matrix product $\t \b$ is 1,
yields the upper bound that $L(\t, \b) \leq L(\W)$, where function
$L(\W)$ is given by
\[L(\W) = \sum_m \sum_i w_{mi}\log \bar w_{mi}.\]
We can establish a tighter bound, which will prove useful
(the proof of this and other technical results are in the \hyperref[supp]{Supplement}):
\begin{prop}\label{prop1}
Given a fixed topic polytope $B$ and $\t$. Let
$U_m$ be the set of words present in document $m$, and assume that
$p_{mi}>0$ $\forall$ $i\in U_m$, then
\begin{eqnarray*}
L(\W) - \frac{1}{2}\sum_{m=1}^{M} N_m\sum_{i\in U_m}(\bar{w}_{mi} - p_{mi})^2 \geq L(\t, \b) \geq L(\W) - \sum_{m=1}^{M} N_m \sum_{i\in U_m}\frac{1}{p_{mi}}(\bar w_{mi} - p_{mi})^2.
\end{eqnarray*}
\end{prop}
Since $L(\W)$ is constant,
the proposition above shows that maximizing the likelihood has the effect
of minimizing the following quantity with respect to both $\t$ and $\b$:
\[\sum_m N_m\sum_{i}(\bar{w}_{mi} - p_{mi})^2.\]

For each fixed $\b$ (and thus $B$), minimizing first with respect
to $\t$ leads to the following
\begin{eqnarray}
\label{theta-solve}
G(B) & := & \min\limits_{\t}\sum_m N_m\sum_{i}(\bar{w}_{mi} - p_{mi})^2 = \sum_{m=1}^M N_m\min\limits_{x:x\in B}\|x-\bar{w}_m\|_2^2,
\end{eqnarray}
where the second equality in the above display is due
$p_m=\sum_k \theta_{mk} \beta_k \in B$.
The proposition suggests a strategy for parameter estimation:
$\b$ (and $B$) can be estimated by minimizing the geometric loss function $G$:
\begin{equation}
\label{geometric-min}
\min_{B} G(B) = \min_{B} \sum_{m=1}^M N_m\min\limits_{x:x\in B}\|x-\bar{w}_m\|_2^2.
\end{equation}
In words, we aim to find a convex polytope $B\in \Delta^{V-1}$, which is
closest to the normalized word counts $\bar w_m$ of the observed documents. It is interesting
to note the presence of document length $N_m$, which provides the weight for
the squared $\ell_2$ error for each document. Thus, our loss function adapts
to the varying length of documents in the collection.
Without the weights, our objective is similar to the sum of squared errors of the Nonnegative Matrix Factorization(NMF). 
\citet{ding2006nonnegative} studied the relation between the likelihood function of interest and NMF, but with a different objective of the NMF problem and without geometric considerations.
Once $\hat{B}$ is solved, $\hat{\t}$ can be obtained as the
barycentric coordinates of the projection of $\bar{w}_m$ onto $\hat{B}$ for each document
$m=1,\ldots,M$ (cf. Eq~\eqref{theta-solve}).
We note that if $K \leq V$, then $B$ is a simplex and $\beta_1,\ldots,\beta_k$
in general positions are the extreme points of $B$, and the barycentric coordinates are unique.
(If $K>V$, the uniqueness no longer holds).
Finally, $\hat p_m = \hat{\theta}_m^T\hat{\b}$ gives the cartesian coordinates of a point
in $B$ that minimizes Euclidean distance to the maximum likelihood estimate:
$\hat{p}_m=\argmin\limits_{x\in B}\|x-\bar{w}_m\|_2$.
This projection is not available in the closed form, but a fast algorithm is available
\citep{golubitsky2012algorithm}, which
can easily be extended to find the corresponding distance and to evaluate our
geometric objective.
%

\subsection{Geometric Dirichlet Means algorithm}
We proceed to devise a procedure for approximately solving the topic polytope $B$ via
Eq.~\eqref{geometric-min}: first, obtain an estimate of the underlying
subspace based on weighted k-means clustering and then, estimate the
vertices of the polytope
that lie on the subspace just obtained via a geometric correction technique. Please refer to the \hyperref[supp]{Supplement} for a clarification of the concrete connection between our geometric loss
function and other objectives which arise in subspace learning and
weighted k-means clustering literature,
the connection that motivates the first step of our algorithm.

\paragraph{Geometric Dirichlet Means (GDM) algorithm}
estimates a topic polytope  $B$ based on the training documents
(see Algorithm \ref{algo}).
The algorithm is conceptually simple, and consists of two main steps:
First, we perform a (weighted) k-means clustering on
the $M$ points $\bar w_1,\ldots, \bar w_M$ to obtain the $K$ centroids
$\mu_1,\ldots, \mu_K$, and second, construct a ray emanating from a (weighted) center of the polytope
and extending through each of the centroids $\mu_k$ until it intersects with a
sphere of radius $R_k$ or with the simplex $\Delta^{V-1}$ (whichever comes first).
The intersection point will be our estimate for vertices $\beta_k$, $k=1,\ldots,K$
of the polytope $B$.  The center $C$ of the sphere is given in step 1 of the algorithm,
while
$R_k = \max\limits_{1 \leq m \leq M} \|C- \bar{w}_m\|_2$, where
the maximum is taken over those documents $m$ that are clustered with label $k$.
\begin{algorithm}[ht]
\caption{Geometric Dirichlet Means (GDM)}
\label{algo}
\begin{algorithmic}[1]
\REQUIRE documents $w_1,\ldots,w_M$, $K$, \\ extension scalar parameters $m_1,\ldots,m_K$
\ENSURE topics $\beta_1,\ldots,\beta_K$
\STATE $C = \frac{1}{M}\sum_m \bar w_m$ \COMMENT{find center of the data}
\STATE $\mu_1,\ldots,\mu_K$ = weighted k-means$(\bar{w}_1,\ldots,{\bar w}_M, K)$ \COMMENT{find centers of $K$ clusters}.
\FORALL{$k=1,\ldots,K$}
\STATE $\beta_k = C+m_k\left(\mu_k-C\right)$.
\IF[threshold topic if it is outside vocabulary simplex $\Delta^{V-1}$]{any $\beta_{ki} <0$}
\FORALL{$i=1,\ldots,V$}
\STATE $\beta_{ki} = \frac{\beta_{ik}\mathds{1}_{\beta_{ki}>0}}{\sum_i\beta_{ki}\mathds{1}_{\beta_{ki}>0}}$.
\ENDFOR
\ENDIF
\ENDFOR
\STATE $\beta_1,\ldots,\beta_K$.
\end{algorithmic}
\end{algorithm}
To see the intuition behind the algorithm, let us consider a simple simulation experiment.
We use the LDA data generative model with $\alpha=0.1$, $\eta=0.1$, $V=5$, $K=4$, $M=5000$, $N_m=100$.
\begin{figure}[h]
\begin{center}
\centerline{\includegraphics[width=0.5\textwidth, height=150pt]{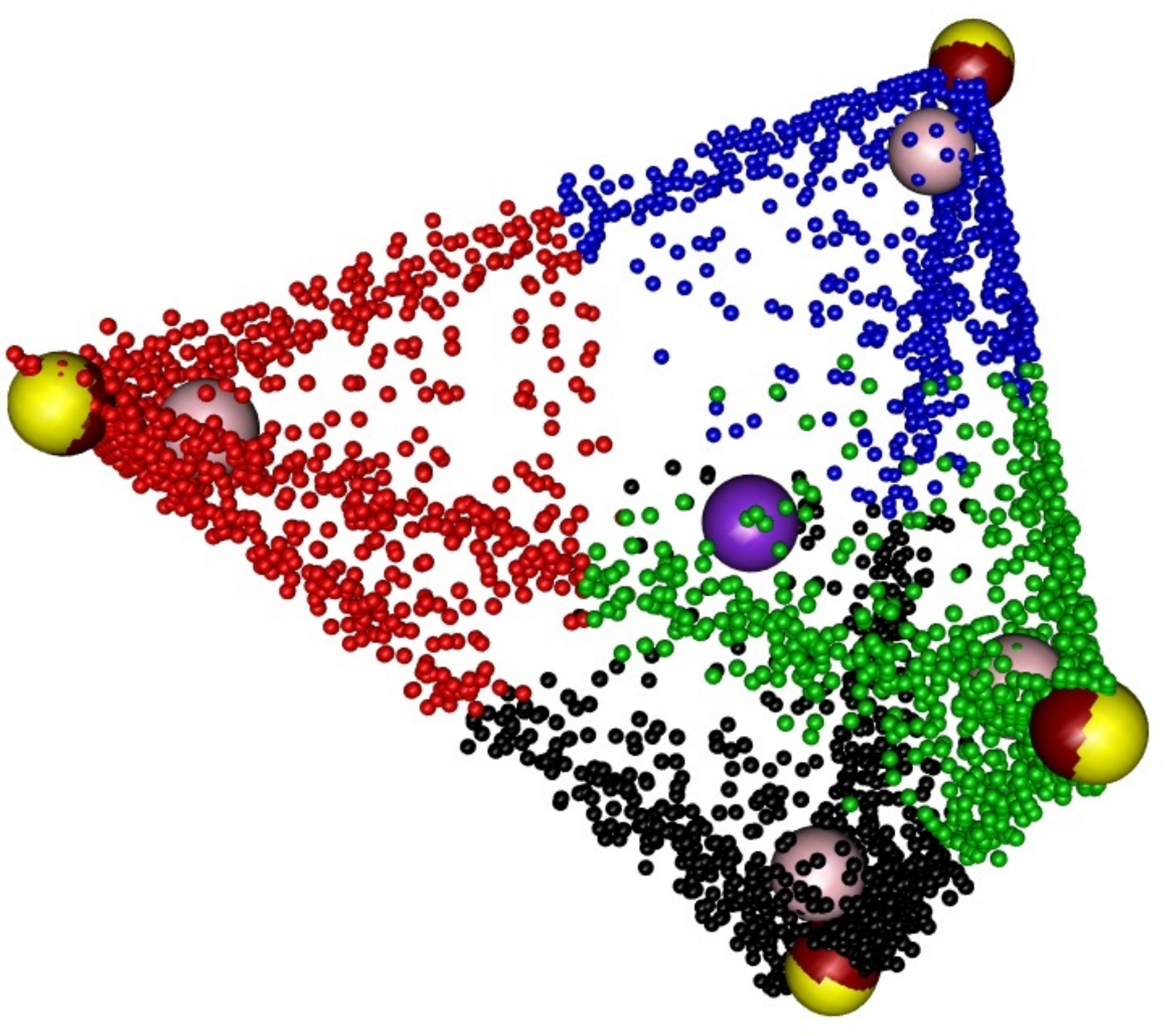}}
\caption{Visualization of GDM: Black, green, red and blue are cluster assignments; purple is the center, pink are cluster centroids, dark red are estimated topics and yellow are the true topics.}
\label{fig:mds}
\end{center}
\vskip -0.2in
\end{figure}
\begin{comment}
\end{comment}
Multidimensional scaling is used for visualization (Fig. \ref{fig:mds}).
We observe that the k-means centroids (pink) do not represent the topics very well,
but our geometric modification finds extreme points of the tetrahedron:
red and yellow spheres overlap, meaning we found the true topics. In this example,
we have used a very small vocabulary size, but in practice $V$ is much higher and
the cluster centroids are often on the boundary of the vocabulary simplex,
therefore we have to threshold the betas at 0.
Extending length until $R_k$ is our default choice for the extension parameters:
\begin{align}
m_k = \frac{R_k}{\|C- \mu_k\|_2}\text{ for } k=1,\ldots,K\label{eq4},
\end{align}
but we will see in our experiments that a careful tuning of the extension
parameters based on optimizing the geometric objective ~\eqref{geometric-min}
over a small range of
$m_k$ helps to improve the performance considerably. We call this
{\bf tGDM} algorithm (tuning details are presented in the \hyperref[supp]{Supplement}).
The connection between extension parameters and the thresholding is the following:
if the cluster centroid assigns probability to a word smaller than the whole data
does on average, this word will be excluded from topic $k$ with large enough $m_k$.
Therefore, the extension parameters can as well be used to control for the sparsity
of the inferred topics.
\subsection{Consistency of Geometric Dirichlet Means}
We shall present a theorem which provides a theoretical justification for the
Geometric Dirichlet Means algorithm. In particular, we will show
that the algorithm can achieve consistent estimates of the topic polytope,
under conditions on the parameters of the Dirichlet distribution of the topic
proportion vector $\theta_m$, along with conditions on the geometry of the convex polytope
$B$. The problem of estimating vertices of a convex polytope given data
drawn from the interior of the polytope has long been a subject of convex geometry ---
the usual setting in this literature is to assume the uniform distribution
for the data sample.
Our setting is somewhat more general --- the distribution of the points
inside the polytope will be driven by a symmetric Dirichlet distribution
setting, i.e.,  $\theta_m \stackrel{iid}{\sim} \textrm{Dir}_{K}(\alpha)$.
(If $\alpha =1$ this results in the uniform distribution on $B$.)
Let $n=K-1$. Assume that the document multinomial parameters $p_1,\ldots,p_M$
(given in Eq.~\eqref{pmi}) are the actual data.
Now we formulate a geometric problem linking the population version of k-means and polytope estimation:
\begin{problem}\label{prob1}
Given a convex polytope $A\in\mathbb{R}^n$, a continuous probability density function
$f(x)$ supported by $A$, find a $K$-partition $A = \bigsqcup\limits_{k=1}^K A_k$
that minimizes:
$$\mathlarger{\sum}_{k}^{K} \int_{A_k}\|\mu_k-x\|^2_2 f(x) \mathop{dx},$$
where $\mu_k$ is the center of mass of $A_k$:
$\mu_k := \frac{1}{\int_{A_k}f(x)\mathop{dx}}\int_{A_k}xf(x)\mathop{dx}$.
\end{problem}
This problem is closely related to the Centroidal Voronoi Tessellations~\citep{du1999centroidal}.
This connection can be exploited to show that
\begin{lem}\label{lem1}
Problem \ref{prob1} has a unique global minimizer.
\end{lem}


In the following lemma, a median of a simplex is a line segment joining
a vertex of a simplex with the centroid of the opposite face.
\begin{lem}\label{lem2}
If $A\in \mathbb{R}^n$ is an equilateral simplex with
symmetric Dirichlet density $f$ parameterized by $\alpha$,
then the optimal centers of mass of the Problem \ref{prob1} lie on the corresponding medians of $A$.
\end{lem}


Based upon these two lemmas,
consistency is established under two distinct asymptotic regimes.
\begin{thm}\label{th1} Let $B=\conv(\beta_1,\ldots,\beta_K)$ be the true convex
polytope from which the $M$-sample $p_1,\ldots,p_M \in \Delta^{V-1}$
are drawn via Eq.~\eqref{pmi}, where $\theta_{m} \stackrel{iid}{\sim}
\textrm{Dir}_K(\alpha)$ for $m=1,\ldots,M$.
\begin{itemize}
\item [(a)] If $B$ is also an equilateral simplex,
then topic estimates obtained by the GDM algorithm
using the extension parameters given in Eq. \eqref{eq4}
converge to the vertices of $B$ in probability, as $\alpha$ is fixed and $M\rightarrow \infty$.
\item [(b)] If $M$ is fixed, while $\alpha \rightarrow 0$ then the
topic estimates obtained by the GDM also converge to the vertices of
$B$ in probability.
\end{itemize}
\end{thm}
\subsection{nGDM: nonparametric geometric inference of topics}
\label{nGDM}
In practice, the number of topics $K$ may be unknown, necessitating a nonparametric
probabilistic approach such as the well-known
Hierarchical Dirichlet Process (HDP) \citep{teh2006hierarchical}.
Our geometric approach can be easily extended to this situation.
The objective \eqref{geometric-min} is now given by
\begin{eqnarray}
\label{geom-hdp}
\min_{B} G(B) = \min_{B} \sum_{m=1}^M N_m\min\limits_{x\in B}\|x-\bar{w}_m\|_2^2 +
\lambda |B|,
\end{eqnarray}
where $|B|$ denotes the number of extreme points of convex polytope $B =
\conv(\beta_1,\ldots,\beta_K)$. Accordingly,
our nGDM algorithm now consists of two steps:
(i) solve a penalized and weighted $k$-means
clustering to obtain the cluster centroids (e.g. using DP-means \citep{kulis2012revisiting}); (ii) apply geometric correction for recovering the extreme points, which proceeds as before.
Our theoretical analysis can be also extended to this nonparametric framework.
We note that the penalty term is reminiscent of the DP-means algorithm of
\citet{kulis2012revisiting}, which was derived under a small-variance asymptotics regime. For the
HDP this corresponds to $\alpha \rightarrow 0$ --- the regime in part (b) of Theorem \ref{th1}.
This is an unrealistic assumption in practice. Our geometric correction
arguably enables the accounting of the non-vanishing variance in data.
We perform a simulation experiment for varying values of $\alpha$ and show that
nGDM outperforms the KL version of DP-means \citep{jiang2012small} in terms of perplexity. This result
is reported in the \hyperref[supp]{Supplement}.

\section{Performance evaluation}
\label{pfm}
\paragraph{Simulation experiments}
We use the LDA model to simulate data and focus our attention on the perplexity of held-out data and minimum-matching Euclidean distance between the true and estimated topics \citep{tang2014understanding}. We explore settings with varying document lengths ($N_m$ increasing from 10 to 1400 - Fig. \ref{fig:mm}(a) and Fig. \ref{fig:pp}(a)), different number of documents ($M$ increasing from 100 to 7000 - Fig. \ref{fig:mm}(b) and Fig. \ref{fig:pp}(b)) and when lengths of documents are small, while number of documents is large ($N_m = 50$, $M$ ranging from 1000 to 15000 - Fig. \ref{fig:mm}(c) and Fig. \ref{fig:pp}(c)). This last setting is of particular interest, since it is the most challenging for our algorithm,
which in theory works well given long documents,
but this is not always the case in practice.
We compare two versions of the Geometric Dirichlet Means algorithm: with tuned extension parameters (tGDM) and the default one (GDM) (cf. Eq. \ref{eq4}) against the
{\bf variational EM} (VEM) algorithm~\citep{blei2003latent} (with tuned hyperparameters), {\bf collapsed Gibbs sampling}
\citep{griffiths2004finding} (with true data generating hyperparameters), and {\bf RecoverKL}
\citep{arora2012practical} and verify the theoretical upper bounds for
topic polytope estimation (i.e. either $(\log M/M)^{0.5}$ or $(\log N_m/N_m)^{0.5}$) -
cf. \citet{tang2014understanding} and \citet{nguyen2015posterior}.
We are also interested in estimating each document's topic
proportion via the projection technique.
RecoverKL produced only a topic matrix,
which is combined with our projection based estimates to compute the perplexity
(Fig. \ref{fig:pp}). Unless otherwise specified, we set
$\eta = 0.1$, $\alpha=0.1$, $V=1200$, $M=1000$, $K=5$; $N_m=1000$ for each $m$;
the number of held-out documents is 100; results are averaged over 5 repetitions. Since finding exact solution to the k-means objective is NP hard, we use the algorithm of \citet{hartigan1979algorithm} with 10 restarts and the k-means++ initialization.
Our results show that (i) Gibbs sampling and tGDM have the best and almost identical performance in terms of statistical estimation; (ii) RecoverKL and GDM are the fastest while
sharing comparable statistical accuracy; (iii) VEM is the worst in most scenarios
due to its instability (i.e. often producing poor topic estimates); (iv) short document lengths
(Fig. \ref{fig:mm}(c) and Fig. \ref{fig:pp}(c)) do not degrade performance of GDM,
(this appears to be an effect of the law of large numbers,
as the algorithm relies on the cluster means, which are obtained by averaging over a large
number of documents);
(v) our procedure for estimating document topic proportions results in a good
quality perplexity of the RecoverKL algorithm in all scenarios (Fig. \ref{fig:pp}) and could be potentially utilized by other algorithms. Additional simulation experiments are presented in the \hyperref[supp]{Supplement}, which considers settings with varying $N_m$,
$\alpha$ and the nonparametric extension.
\begin{figure*}[ht]
\vskip 0.2in
\begin{center}
\centerline{\includegraphics[width=\textwidth]{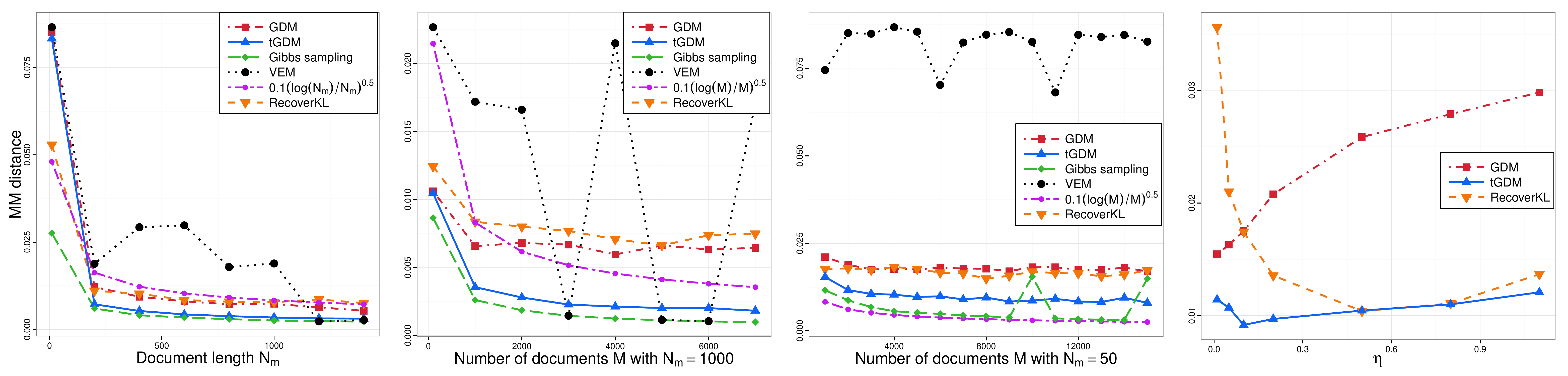}}
\caption{Minimum-matching Euclidean distance: increasing $N_m$, $M=1000$ (a); increasing $M$, $N_m=1000$ (b); increasing $M$, $N_m=50$ (c); increasing $\eta$, $N_m=50$, $M=5000$ (d).}
\label{fig:mm}
\end{center}
\vskip -0.2in
\end{figure*}
\begin{figure*}[ht]
\begin{center}
\centerline{\includegraphics[width=\textwidth]{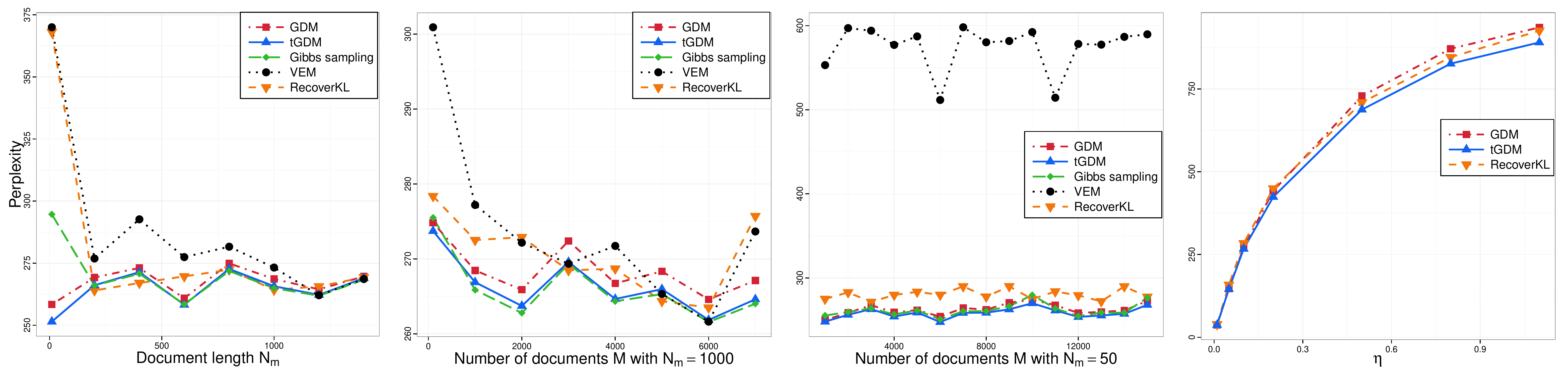}}
\caption{Perplexity of the held-out data: increasing $N_m$, $M=1000$ (a); increasing $M$, $N_m=1000$ (b); increasing $M$, $N_m=50$ (c); increasing $\eta$, $N_m=50$, $M=5000$ (d).}
\label{fig:pp}
\end{center}
\vskip -0.2in
\end{figure*}
\paragraph{Comparison to RecoverKL}
Both tGDM and RecoverKL exploit the geometry
of the model, but they rely on very different assumptions:
RecoverKL requires the presence of anchor words in the topics
and exploits this in a crucial way \citep{arora2012practical}; our method relies on
long documents in theory, even though the violation of this does not appear
to degrade its performance in practice, as we have shown earlier.
The comparisons are performed by varying the document length $N_m$,
and varying the Dirichlet parameter $\eta$
(recall that $\beta_k|\eta\,\thicksim\, \text{Dir}_V(\eta)$).
In terms of perplexity, RecoverKL, GDM and tGDM perform similarly
(see Fig.\ref{fig:ar}(c,d)), with a slight edge to tGDM. Pronounced
differences come in the quality of topic's word distribution estimates.
To give RecoverKL the advantage, we considered manually inserting anchor words for
each topic generated, while keeping the document length
short, $N_m = 50$ (Fig. \ref{fig:ar}(a,c)). We found that tGDM outperforms
RecoverKL when $\eta \leq 0.3$, an arguably more common setting,
while RecoverKL is more accurate when $\eta \geq 0.5$.
However, if the presence of anchor words is not explicitly enforced,
tGDM always outperforms RecoverKL in terms of topic distribution estimation accuracy
for all $\eta$ (Fig. \ref{fig:mm}(d)). The superiority of tGDM persists even as
$N_m$ varies from 50 to 10000 (Fig. \ref{fig:ar}(b)), while GDM is comparable
to RecoverKL in this setting.

\begin{figure*}[ht]
\begin{center}
\centerline{\includegraphics[width=\textwidth]{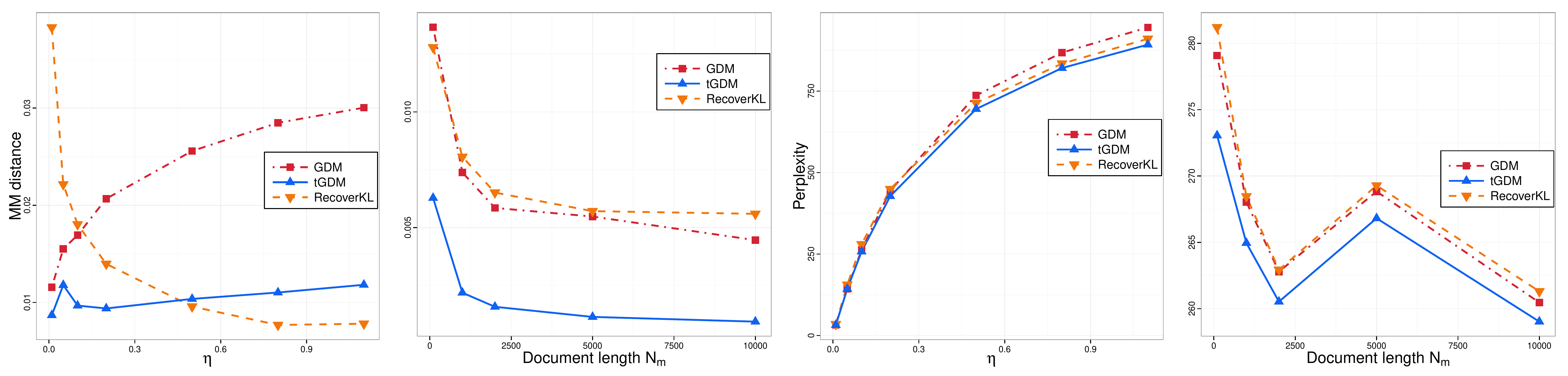}}
\caption{MM distance and Perplexity for varying $\eta$, $N_m=50$ with anchors (a,c);
varying $N_m$ (b,d).}
\label{fig:ar}
\end{center}
\vskip -0.2in
\end{figure*}
\paragraph{NIPS corpora analysis}
We proceed with the analysis of the NIPS corpus.\footnote{\url{https://archive.ics.uci.edu/ml/datasets/Bag+of+Words}}
After preprocessing, there are 1738 documents and 4188 unique words. Length of documents ranges from 39 to 1403 with mean of 272.
We consider $K=5, 10, 15, 20$, $\alpha=\frac{5}{K}$, $\eta=0.1$. For each value of $K$ we set aside 300 documents chosen at random to compute the perplexity and average results over 3 repetitions.
Our results are compared against Gibbs sampling, Variational EM and RecoverKL (Table \ref{nips}). For $K=10$, GDM with 1500 k-means iterations and 5 restarts in R took 50sec; Gibbs sampling with 5000 iterations took 10.5min; VEM with 750 variational, 1500 EM iterations and 3 restarts took 25.2min; RecoverKL coded in Python took 1.1min. We note that with recent developments (e.g., \citep{hoffman2013stochastic}) VEM could be made faster, but its statistical accuracy remains poor.
Although RecoverKL is as fast as GDM, its perplexity performance
is poor and is getting worse with more topics, which we believe could be due to lack of anchor words in the data. We present topics found by Gibbs sampling, GDM and RecoverKL for
$K=10$ in the \hyperref[supp]{Supplement}.
\begin{table}[ht]
\caption{Perplexities of the 4 topic modeling algorithms trained on the NIPS dataset.}
\label{nips}
\centering
\begin{tabular}{lcccc}
    \toprule
    & GDM   & RecoverKL & VEM &  Gibbs sampling \\
\midrule
$K=5$    & 1269 &  1378 & 1980 &  1168 \\
$K=10$   & 1061 &  1235 & 1953 &  924  \\
$K=15$   & 957  &  1409 & 1545 &  802  \\
$K=20$   & 763  &  1586 & 1352 &  704  \\
\bottomrule
\end{tabular}
\end{table}


\section{Discussion}
\label{dis}
We wish to highlight a conceptual aspect of GDM distinguishing it from
moment-based methods such as RecoverKL. GDM operates on the
document-to-document distance/similarity matrix, as opposed to the
second-order word-to-word matrix. So, from an optimization viewpoint,
our method can be viewed as the dual to RecoverKL method, which must
require anchor-word assumption to be computationally feasible and
theoretically justifiable. While the computational complexity of
RecoverKL grows with the vocabulary size and not the corpora size,
our convex geometric approach continues to be computationally feasible
when number of documents is large: since only documents near the
polytope boundary are relevant in the inference of the extreme
points, we can discard most documents residing near the polytope's center.

We discuss some potential improvements and extensions next.
The tGDM algorithm showed a superior performance when
the extension parameters are optimized.
This procedure, while computationally effective relative to methods such
as Gibbs sampler, may still be not scalable to massive datasets. It seems
possible to reformulate the geometric objective as a function of
extension parameters, whose optimization can be performed more efficiently.
In terms of theory, we would like to establish the error bounds by exploiting
the connection of topic inference to the geometric problem of
Centroidal Voronoi Tessellation of a convex polytope.

The geometric approach to topic modeling and inference
may lend itself naturally to other LDA extensions, as we have
demonstrated with nGDM algorithm for the HDP \citep{teh2006hierarchical}.
Correlated topic models of \cite{blei2006correlated}
also fit naturally into the geometric framework --- we
would need to adjust geometric modification to capture logistic normal
distribution of topic proportions inside the topic polytope.
Another interesting direction is to consider
dynamic \citep{blei2006dynamic} (extreme points of topic polytope evolving
over time) and supervised \citep{mcauliffe2008supervised} settings.
Such settings appear relatively more challenging, but they are
worth pursuing further.

\subsubsection*{Acknowledgments}
This research is supported in part by grants NSF CAREER DMS-1351362 and NSF CNS-1409303.

\appendix
\section{Supplementary material}
\label{supp}
\subsection{Proof of Proposition \ref{prop1}}
\begin{proof}
Consider the KL divergence between two distributions
parameterized by $\bar w_m$ and $p_m$, respectively:
\begin{align*}
D&(P_{\bar{w}_m}\|P_{p_m}) = \sum_{i\in U_m} \bar{w}_{mi} \log{\frac{\bar{w}_{mi}}{p_{mi}}}\\
&=\frac{1}{N_m}\left(\sum_{i\in U_m} w_{mi} \log \bar{w}_{mi} - \sum_{i\in U_m} w_{mi} \log p_{mi}\right).
\end{align*}
Then $L(\W) - L(\t, \b) = \sum_m N_m D(P_{\bar{w}_m}\|P_{p_m}) \geq 0$, due to the non-negativity of
KL divergence. Now we shall appeal to a standard lower bound for the KL divergence
\citep{cover2012elements}:
\begin{align*}
D(P_{\bar{w}_m}\|P_{p_m}) 
& \geq \frac{1}{2}\sum_{i\in U_m}(\bar{w}_{mi} - p_{mi})^2,
\end{align*}
and an upper bound via $\chi^2$-distance (e.g. see \citet{sayyareh2011new}):
\begin{eqnarray*}
D(P_{\bar{w}_m}\|P_{p_m}) \leq
\sum_{i\in U_m}\frac{1}{p_{mi}}(\bar{w}_{mi} - p_{mi})^2.
\end{eqnarray*}
Taking summation of both bounds over $m=1,\ldots, M$ concludes the proof.
\end{proof}

\subsection{Connection between our geometric loss function and
other objectives which arise in subspace learning and
k-means clustering problems.}
Recall that our geometric objective (Eq. \ref{geometric-min}) is:
\begin{equation}
\min_{B} G(B) = \min_{B} \sum_{m=1}^M N_m\min\limits_{x:x\in B}\|x-\bar{w}_m\|_2^2.\nonumber
\end{equation}
We note that this optimization problem can be reduced to two other well-known problems when the objective
function and constraints are suitably relaxed/modified:
\begin{itemize}
\item A version of weighted low-rank matrix approximation is $\min\limits_{\rank(\hat D)\leq r} \tr((\hat D - D)^T Q (\hat D - D))$. If $Q=\diag(N_1,\ldots,N_M)$, $D=\W$, $r=K$ and $\hat D = \t\b$, the problem looks similar to the geometric objective without constraints and has a closed form solution \citep{manton2003geometry}: $\hat D = Q^{-1/2}U\Sigma_K V^T$, where
    \begin{align}
    Q^{1/2}D = U \Sigma_K V^T \label{eq1}
    \end{align}
    is the singular value decomposition and $\Sigma_K$ is the truncation to $K$ biggest singular values. Also note that here and further without loss of generality we assume $M \geq V$, if $M < V$ for the proofs to hold we replace $Q^{1/2}D$ with $(Q^{1/2}\W)^T$.
\item The k-means algorithm involves optimizing the objective
\citep{hartigan1979algorithm,lloyd1982least,macqueen1967some}:
$\min\limits_{x_1,\ldots,x_K}\sum_m \min\limits_{i\in\{1,\ldots,K\}}\|\bar{w}_m - x_i\|^2_2$.
Our geometric objective ~\eqref{geometric-min} is quite similar --- it replaces the second minimization with minimizing over
the convex hull of $\{x_1,\ldots,x_K\}$ and includes weight $N_m$s.
\item The two problems described above are connected in the following way
\citep{xu2003document}. Define the weighted k-means objective with
respect to cluster assignments: $\sum_k\sum_{m\in C_k}N_m\|\bar{w}_m-\mu_k\|^2$,
where $\mu_k$ is the centroid of the $k$-th cluster:
\begin{align}
\mu_k = \frac{\sum_{m \in C_k} N_m\bar{w}_m}{\sum_{m \in C_k} N_m}. \label{eq2}
\end{align}
Let $S_k$ be the optimal indicator vector of cluster $k$, i.e., $m$-th element is 1 if $m \in C_k$ and 0 otherwise. Define
\begin{align}
Y_k = \frac{Q^{1/2}S_k}{\|Q^{1/2}S_k\|^2_F}. \label{eq3}
\end{align}
If we relax the constraint on $S_k$ to allow any real values instead of only binary
values, then $Y$ can be solved via the following eigenproblem:
$Q^{1/2} \W \W^T Q^{1/2} Y = \lambda Y$.
\end{itemize}
Let us summarize the above observations by the following:
\begin{prop}\label{prop2}
Given the $M\times V$ normalized word counts matrix $\W$. Let $\mu_1,\ldots, \mu_K$ be
the optimal cluster centroids of the weighted k-means problem given by Eq.~\eqref{eq2},
and let $v_k$s be the columns of $V$ in the SVD of Eq.~\eqref{eq1}.
Then,
\[\spa(\mu_1,\ldots,\mu_K) = \spa(v_1,\ldots,v_K).\]
\end{prop}
\begin{proof}
Following \citet{ding2004k}, let $P_c$ be an operator projecting any vector onto $\spa(\mu_1,\ldots,\mu_K)$:
$P_c = \sum_k \mu_k \mu_k^T$.
Recall that $S_k$ is the indicator vector of cluster $k$ and $Y_k$ defined in Eq. (\ref{eq3}). Then
$\mu_k = \frac{\W^T Q S_k}{\|Q^{1/2}S_k\|^2_F} = \W^T Q^{1/2} Y_k$, and
$P_c = \sum_k \W^T Q^{1/2} Y_k (\W^T Q^{1/2} Y_k)^T$.
Now, note that $Y_k$'s are the eigenvectors of $Q^{1/2} \W \W^T Q^{1/2}$, which are also left-singular vectors of $Q^{1/2}\W = U \Sigma V^T$, so
\begin{align*}
P_c & = (Q^{1/2}\W)^T Y_k ((Q^{1/2}\W)^T Y_k)^T = \sum_k \lambda_k^2 v_k v_k^T,
\end{align*}
which is the projection operator for $\spa(v_1,\ldots,v_K)$. Hence, the two subspaces are equal.
\end{proof}
Prop.~\ref{prop2} and the preceding discussions motivate the GDM algorithm for
estimating the topic polytope: first, obtain an estimate of the underlying
subspace based on k-means clustering and then, estimate the vertices of the polytope
that lie on the subspace just obtained.

\subsection{Proofs of technical lemmas}
Recall Problem \ref{prob1} from the main part:
\begin{pmn}[\ref{prob1}]
Given a convex polytope $A\in\mathbb{R}^n$, a continuous probability density function
$f(x)$ supported by $A$, find a $K$-partition $A = \bigsqcup\limits_{k=1}^K A_k$
that minimizes:
$$\mathlarger{\sum}_{k}^{K} \int_{A_k}\|\mu_k-x\|^2_2 f(x) \mathop{dx},$$
where $\mu_k$ is the center of mass of $A_k$:
$\mu_k := \frac{1}{\int_{A_k}f(x)\mathop{dx}}\int_{A_k}xf(x)\mathop{dx}$.
\end{pmn}
\paragraph{Proof of Lemma \ref{lem1}}
\begin{proof}
The proof follows from a sequence of results of \citet{du1999centroidal}, which we
now summarize.
First, if the $K$-partition $(A_1,\ldots,A_K)$ is a minimizer of Problem \ref{prob1}, then $A_k$s are
the Voronoi regions corresponding to the $\mu_k$s. Second,
Problem \ref{prob1} can be restated in terms of the $\mu_k$s to minimize
$\mathcal{K}(\mu_1,\ldots,\mu_K) = \mathlarger{\sum}_k \int_{\hat A_k}\|\mu_k-x\|^2_2 f(x) \mathop{dx}$,
where $\hat A_k$s are the Voronoi regions corresponding to their centers of mass $\mu_k$s.
Third, $\mathcal{K}(\mu_1,\ldots,\mu_K)$ is a continuous function and admits a global minimum.
Fourth, the global minimum is unique if the distance function in $\mathcal{K}$ is
strictly convex and the Voronoi regions are convex.
Now, it can be verified that the squared Euclidean distance is strictly convex. Moreover,
Voronoi regions are intersections of half-spaces with the convex polytope $A$, which
can also be represented as an intersection of half-spaces. Therefore,
the Voronoi regions of Problem \ref{prob1} are convex polytopes, and it follows that
the global minimizer is unique.
\end{proof}

\paragraph{Proof of Lemma \ref{lem2}}
\begin{proof}
Since $f$ is a symmetric Dirichlet density, the
center of mass of $A$ coincides with its centroid.
Let $n=3$. In an equilateral triangle, the centers of mass
$\mu_1,\mu_2,\mu_3$ form an equilateral triangle $C$.
An intersection point of the Voronoi regions $A_1,A_2,A_3$ is the circumcenter
and the centroid of $C$, which is also a circumcenter and
centroid of $A$. Therefore, $\mu_1,\mu_2,\mu_3$ are located on the
medians of $A$ with exact positions depending on the $\alpha$.
The symmetry and the property of circumcenter coinciding with centroid
carry over to the general $n$-dimensional equilateral simplex \citep{westendorp2013circum}.
\end{proof}

\subsection{Proof of Theorem \ref{th1}}

\begin{proof}
For part (a), let $(\hat \mu_1,\ldots, \hat \mu_K)$ be the minimizer of
the k-means problem
$\min\limits_{\mu_1,\ldots,\mu_K}\sum_m \min\limits_{i\in\{1,\ldots,K\}}\|p_m - \mu_i\|^2_2.$
Let $\tilde \mu_1,\ldots, \tilde \mu_K$ be the centers of mass of the solution of Problem \ref{prob1}
applied to $B$ and the Dirichlet density.
By Lemma 1, these centers of mass are unique, as they correspond to the unique optimal
$K$-partition.  Accordingly, by the strong consistency of k-means clustering under
the uniqueness condition \citep{pollard1981strong}, as $M\rightarrow \infty$,
\[\textrm{Conv}(\hat \mu_1, \ldots, \hat \mu_K)\rightarrow \textrm{Conv}(\tilde \mu_1,\ldots, \tilde \mu_K) \text{ a.s.},\]
where the convergence is assessed in either Hausdorff or the minimum matching distance
for convex sets \citep{nguyen2015posterior}.
Note that $C = \frac{1}{M}\sum_m p_m$ is a strongly consistent estimate
of the centroid $C_0$ of $B$, by the strong law of large numbers.
Lemma 2 shows that $\tilde \mu_1,\ldots, \tilde \mu_K$ are
located on the corresponding medians. To complete the proof, it remains to
show that $\hat R := \max\limits_{1 \leq m \leq M} \| C- p_m\|_2$ is
a weakly consistent estimate of the circumradius $R_0$ of $B$. Indeed,
for a small $\epsilon > 0$
define the event $E^k_m = \{p_m \in B_{\epsilon}(\beta_k) \cap B\}$, where $B_\epsilon(\beta_k)$
is an $\epsilon$-ball centering at vertex $\beta_k$.
Since $B$ is equilateral and the density over it is symmetric and positive
everywhere in the domain, $\P(E^1_m)=\ldots=\P(E^K_m)=: b_\epsilon>0$.
Let $E_m = \bigcup\limits_k E^k_m$, then $\P(E_m) = b_\epsilon K$. We have
\begin{align*}
& \limsup_{M\rightarrow \infty}\P(|\hat R - R_0| > 2\epsilon)= \limsup_{M\rightarrow \infty}\P(\max\limits_{1 \leq m \leq M}\| C_0- p_m\|_2 <
R_0 - \epsilon) < \\
& < \limsup_{M\rightarrow \infty} \P(\bigcap\limits_{m=1}^M E_m^\complement) =
\limsup_{M\rightarrow \infty} (1-b_\epsilon K)^M  = 0.
\end{align*}
A similar argument allows us to establish that each $R_k$ is also
a weakly consistent estimate of $R_0$. This completes the proof of part (a).
For a proof sketch of part (b),
for each $\alpha>0$, let $(\mu^\alpha_1,\ldots,\mu^\alpha_K)$
denote the $K$ means obtained by the k-means clustering algorithm. It suffices
to show that these estimates converge to the vertices of $B$. Suppose this is not the
case, due to the compactness of $B$, there is a subsequence of the $K$ means, as
$\alpha \rightarrow 0$, that tends to $K$ limit points, some of which are not
the vertices of $B$. It is a standard fact of Dirichlet distributions that
as $\alpha \rightarrow 0$, the distribution of the $p_m$ converges weakly to
the discrete probability measure $\sum_{k=1}^{K}\frac{1}{K} \delta_{\beta_k}$.
So the k-means objective function tends to $\frac{M}{K}\sum_{k}\min_{i\in\{1,\ldots,K\}}
\|\beta_k-\mu^\alpha_i\|_2^2$,
which is strictly bounded away from 0, leading to a contradiction.
This concludes the proof.
\end{proof}

\subsection{Tuned GDM}
In this section we discuss details of the extension parameters tuning. Recall that GDM requires extension scalar parameters $m_1,\ldots,m_K$ as part of its input. Our default choice (Eq. \eqref{eq4}) is
\begin{align*}
m_k = \frac{R_k}{\|C- \mu_k\|_2}\text{ for } k=1,\ldots,K,
\end{align*}
where $R_k = \max\limits_{m \in C_k} \|C- \bar{w}_m\|_2$ and $C_k$ is the set of indices of documents belonging to cluster $k$. In some situations (e.g. outliers making extension parameters too big) tuning of the extension parameters can help to improve the performance, which we called \textbf{tGDM} algorithm. Recall the geometric objective \eqref{geometric-min} and let
\begin{equation}
G_k(B) := \sum_{m\in C_k} N_m\min\limits_{x:x\in B}\|x-\bar{w}_m\|_2^2,
\end{equation}
which is simply the geometric objective evaluated at the documents of cluster $k$. For each $k=1,\ldots,K$ we used line search procedure \citep{brent2013algorithms} optimization of $G_k(B)$ in an interval from 1 up to default $m_k$ as in \eqref{eq4}. Independent tuning for each $k$ gives an approximate solution, but helps to reduce the running time.

\subsection{Performance evaluation}
Here we present some additional simulation results and NIPS topics.
\paragraph{Nonparametric analysis with DP-means.}
Based on simulations we show how nGDM can be used when number of topics is unknown and compare it against DP-means utilizing KL divergence (KL DP-means) by \citet{jiang2012small}. We analyze settings with $\alpha$ ranging from 0.01 to 2. Recall that KL DP-means assumes $\alpha \rightarrow 0$. $V = 1200$, $M = 2000$, $N_m = 3000$, $\eta = 0.1$, true $K=15$. For each value of $\alpha$ average over 5 repetitions is recorded and we plot the perplexity of $100$ held-out documents. Fig. \ref{fig:sva} supports our argument - for small values of $\alpha$ both methods perform equivalently well (KL DP-means due to variance assumption being satisfied and nGDM due to part (b) of Theorem 1), but as $\alpha$ gets bigger, we see how our geometric correction leads to improved performance.
\begin{figure*}[ht]
\vskip 0.2in
\begin{center}
\centerline{\includegraphics[width=0.45\textwidth]{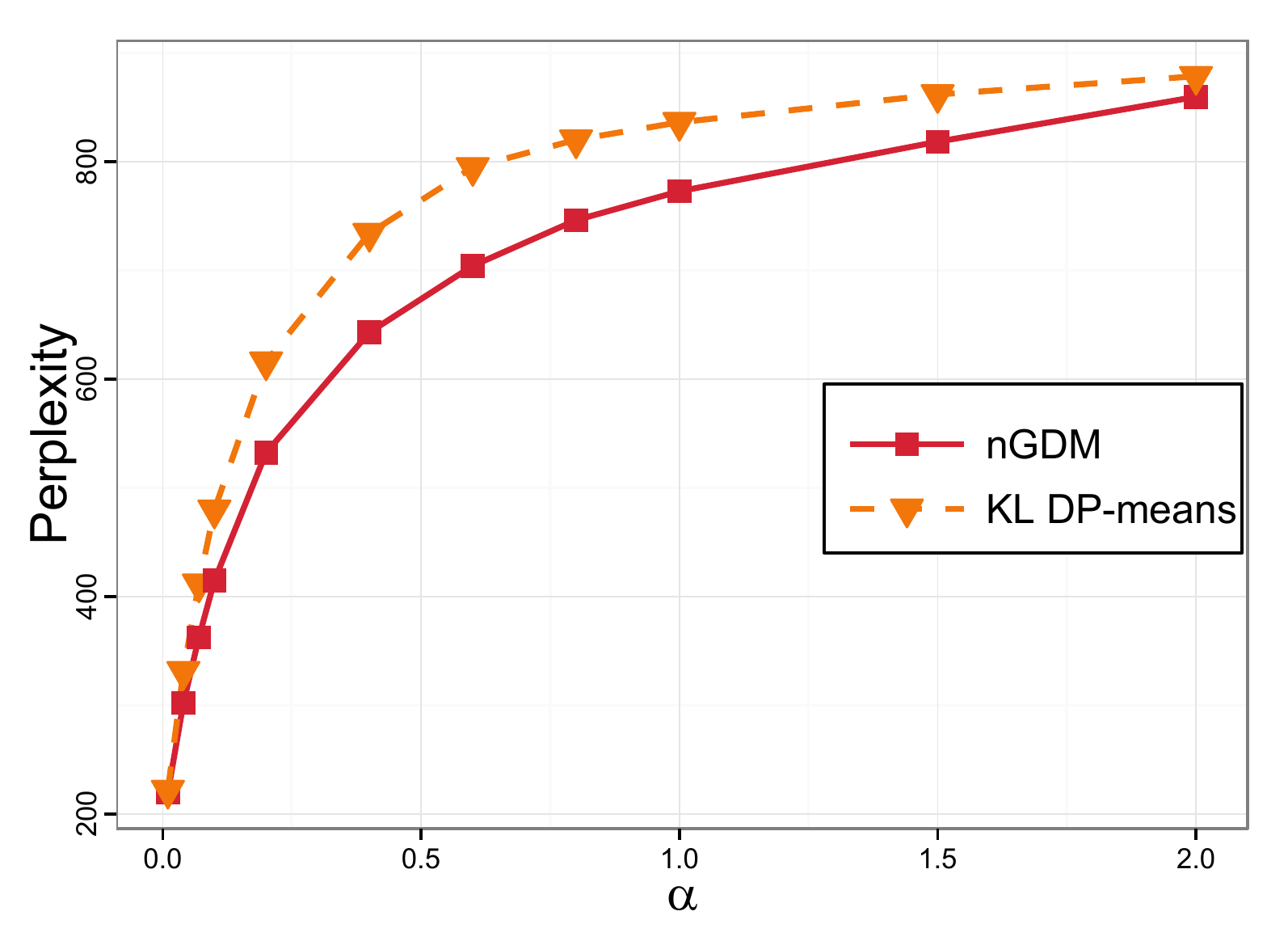}}
\caption{Perplexity for varying $\alpha$}
\label{fig:sva}
\end{center}
\vskip -0.2in
\end{figure*}

\paragraph{Documents of varying size.}
Until this point all documents are of the same length.
Next,  we evaluate the improvement of our method when document length varies.
The lengths are randomly sampled from 50 to 1500 and the experiment is repeated
20 times. The weighted GDM uses document lengths as weights for computing the data center and
training k-means.
In both performance measures (Fig. \ref{fig:3rd} left and center) the
weighted version consistently outperforms the unweighted one, while the
tuned weighted version stays very close to Gibbs sampling results.

\paragraph{Effect of the document topic proportions prior.}
Recall that topic proportions are sampled from the Dirichlet distribution
$\theta_m|\alpha\,\thicksim\,\text{Dir}_K(\alpha)$. We let $\alpha$ increase from 0.01 to 2.
Smaller $\alpha$ implies that samples are close to the extreme points, and hence GDM estimates topics better.
This also follows from Theorem 1(b) of the paper. We see (Fig. \ref{fig:3rd} right) that our solution and Gibbs sampling are almost identical for small $\alpha$, while VEM is unstable. With increased $\alpha$ Gibbs sampling remains the best, while our algorithm remains better than VEM. We also note that increasing $\alpha$ causes error of all
methods to increase.
\begin{figure*}[ht]
\vskip 0.2in
\begin{center}
\centerline{\includegraphics[width=0.8\textwidth]{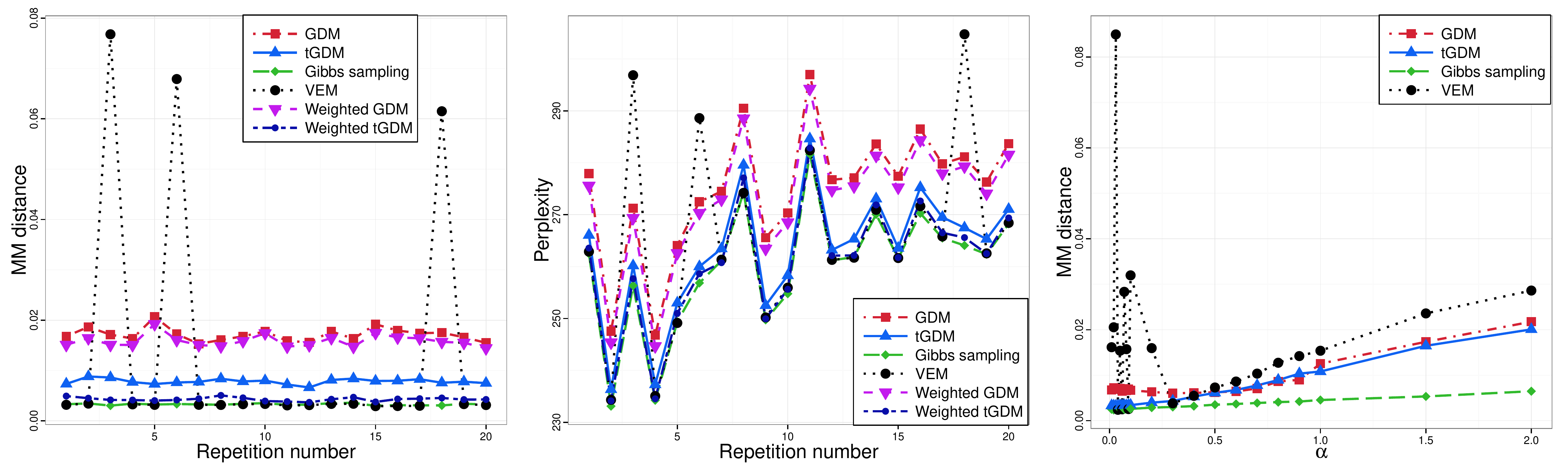}}
\caption{Minimum-matching Euclidean distance: varying $N_m$ (left); increasing $\alpha$ (right). Perplexity for varying $N_m$ (center).}
\label{fig:3rd}
\end{center}
\vskip -0.2in
\end{figure*}
\paragraph{Projection estimate analysis.}
Our objective function \eqref{geometric-min} motivates the estimation of document topic proportions by taking the barycentric coordinates of the projection of the normalized word counts of a document onto the topic polytope.
To do this we utilized the projection algorithm of \citet{golubitsky2012algorithm}.
Note that some algorithms (RecoverKL in particular) do not have a built in method for finding topic proportions of the unseen documents. Our projection based estimate can solve this issue, as it can find topic proportions of a document only based on the topic polytope.
Fig. \ref{fig:proj} shows that perplexity with projection estimates closely follows corresponding results and outperforms VEM on the short documents (Fig. \ref{fig:proj} (right)).
%
\begin{figure*}[ht]
\vskip 0.2in
\begin{center}
\centerline{\includegraphics[width=0.8\textwidth]{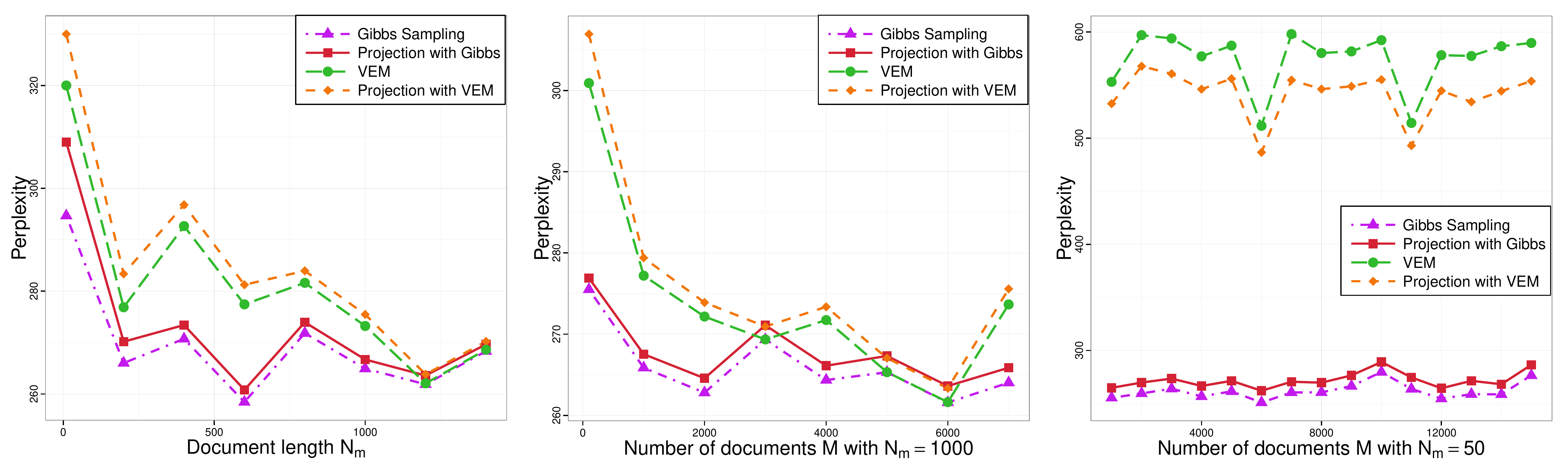}}
\caption{Projection method: increasing $N_m$, $M=1000$ (left); increasing $M$, $N_m=1000$ (center); increasing $M$, $N_m=50$ (right).}
\label{fig:proj}
\end{center}
\vskip -0.2in
\end{figure*}

\begin{table}[ht]
\scriptsize
\caption*{\textbf{Top 10 words (columns) of each of the 10 learned topics of NIPS dataset}}
\label{nips_gdm}
\vskip 0.2in
\centering
\begin{tabular}{llllllllll}
      \toprule
  \multicolumn{10}{c}{GDM topics} \\
  \midrule
analog & regress. & reinforc. & nodes & speech & image & mixture & neurons & energy & rules \\
 circuit & kernel & policy & node & word & images & experts & neuron & characters & teacher \\
memory & bayesian & action & classifier & hmm & object & missing & cells & boltzmann & student \\
 chip & loss & controller & classifiers & markov & visual & mixtures & cell & character & fuzzy \\
 theorem & posterior & actions & tree & phonetic & objects\hspace{5pt} & expert & synaptic & hopfield & symbolic \\
 sources & theorem & qlearning & trees & speaker & face & gating & spike & temperature & saad \\
 polynom. & hyperp. & reward & bayes & acoustic & pixel & posterior & activity & annealing & membership \\
 separation & bounds & sutton & rbf & phoneme & pixels & tresp & firing & kanji & rulebased \\
 recurrent & monte & robot & theorem & hmms & texture & loglikel. & visual & adjoint & overlaps \\
 circuits & carlo & barto & boolean & hybrid & motion & ahmad & cortex & window & children \\
   \bottomrule
\end{tabular}
\vskip 0.2in
\begin{tabular}{llllllllll}
      \toprule
  \multicolumn{10}{c}{Gibbs sampler topics} \\
  \midrule
neurons & rules & mixture & reinforc. & memory & speech & image & analog & theorem & classifier \\
 cells & language & bayesian & policy & energy & word & images & circuit & regress. & nodes \\
 cell & recurrent & posterior & action & neurons & hmm & visual & chip & kernel & node \\
neuron & node & experts & robot & neuron & auditory & object\hspace{10pt} & voltage & loss & classifiers \\
  activity & tree & entropy & motor & capacity & sound & motion & neuron\hspace{10pt} & bounds & tree \\
  synaptic & memory & mixtures & actions & hopfield & phoneme & objects & vlsi & proof & clustering \\
  firing & nodes & markov & controller & associative & acoustic & spatial & circuits & polynom. & character \\
  spike & symbol & separation & trajectory & recurrent & hmms & face & digital & lemma & rbf \\
  stimulus & symbols & sources & arm & attractor & mlp & pixel & synapse & teacher & cluster \\
  cortex & grammar & principal & reward & boltzmann & segment. & pixels & gate & risk & characters \\
   \bottomrule
\end{tabular}
\vskip 0.2in
\begin{tabular}{llllllllll}
      \toprule
  \multicolumn{10}{c}{RecoverKL topics} \\
  \midrule
entropy & reinforc. & classifier & loss & ensemble & neurons & penalty & mixture & validation & image \\
image & controller & classifiers & theorem & energy & neuron & rules & missing & regress. & visual \\
 kernel & policy & speech & bounds & posterior & spike & regress. & recurrent & bayesian & motion \\
energy & action & nodes & proof & bayesian & synaptic & bayesian & bayesian & crossvalid. & cells \\
 ica & actions & word & lemma & speech & cells & energy & posterior & risk & neurons \\
images & memory & node & polynom. & boltzmann & firing & theorem & image & stopping & images \\
 separation & robot & image & neurons & student & cell & analog & markov & tangent & receptive \\
 clustering & trajectory & tree & regress. & face & activity & regulariz. & speech & image & circuit \\
sources & sutton & character & nodes & committee & synapses & recurrent & images & kernel & spatial \\
 mixture & feedback & memory & neuron & momentum & stimulus & perturb. & object & regulariz. & object \\
   \bottomrule
\end{tabular}
\end{table}

\clearpage
\small
\bibliography{reference}
\bibliographystyle{icml2016}

\end{document}